\documentclass[letterpaper]{article} 
\usepackage{aaai2026}  
\usepackage{times}  
\usepackage{helvet}  
\usepackage{courier}  
\usepackage[hyphens]{url}  
\usepackage{graphicx} 
\usepackage{natbib}  
\usepackage{caption} 
\usepackage{algorithm}
\usepackage[noend]{algpseudocode}
\usepackage{amsmath}
\usepackage{amsfonts}
\usepackage{amsthm}
\usepackage{subcaption}

\usepackage{newfloat}
\usepackage{listings}
\DeclareCaptionStyle{ruled}{labelfont=normalfont,labelsep=colon,strut=off} 
\floatstyle{ruled}
\newfloat{listing}{tb}{lst}{}
\floatname{listing}{Listing}
\title{GDBA Revisited: Unleashing the Power of Guided Local Search for \\  Distributed Constraint Optimization}
\author {
    Yanchen Deng\textsuperscript{\rm 1},
    Xinrun Wang\textsuperscript{\rm 2}\thanks{Corresponding author.},
    Bo An\textsuperscript{\rm 1}
}
\affiliations {
    \textsuperscript{\rm 1}College of Computing and Data Science, Nanyang Technological University, Singapore\\
    \textsuperscript{\rm 2}School of Computing and Information Systems, Singapore Management University, Singapore\\
    ycdeng@ntu.edu.sg, xrwang@smu.edu.sg, boan@ntu.edu.sg
}

\begin{document}

\maketitle

\begin{abstract}
Local search is an important class of incomplete algorithms for solving Distributed Constraint Optimization Problems (DCOPs) but it often converges to poor local optima. While Generalized Distributed Breakout Algorithm (GDBA) provides a comprehensive rule set to escape premature convergence, its empirical benefits remain marginal on general-valued problems. In this work, we systematically examine GDBA and identify three factors that potentially lead to its inferior performance, i.e., over-aggressive constraint violation conditions, unbounded penalty accumulation, and uncoordinated penalty updates. To address these issues, we propose Distributed Guided Local Search (DGLS), a novel GLS framework for DCOPs that incorporates an adaptive violation condition to selectively penalize constraints with high cost, a penalty evaporation mechanism to control the magnitude of penalization, and a synchronization scheme for coordinated penalty updates. We theoretically show that the penalty values are bounded, and agents play a potential game in DGLS. Extensive empirical results on various benchmarks demonstrate the great superiority of DGLS over state-of-the-art baselines. Compared to Damped Max-sum with high damping factors, our DGLS achieves competitive performance on general-valued problems, and outperforms by significant margins on structured problems in terms of anytime results.
\end{abstract}

\begin{links}
\link{Code}{https://github.com/ycdeng-ntu/DGLS}
\link{Extended version}{https://arxiv.org/pdf/2508.06899}
\end{links}

\section{Introduction}
Distributed Constraint Optimization Problems (DCOPs) \cite{modi2005adopt,fioretto2018distributed} are a fundamental formalism for cooperative multi-agent systems where a set of autonomous agents coordinate with each other to pursue a global objective via localized communication. DCOPs have been successfully applied to model various real-world applications, including scheduling \cite{pertzovsky2024collision}, resource allocation \cite{monteiro2012multi}, and smart grids \cite{fioretto2017distributed}.

Efficiently solving DCOPs remains a long-standing challenge due to the inherent NP-hardness. Over the past decades, many DCOP algorithms have been proposed and are generally categorized into \emph{complete} algorithms and \emph{incomplete} algorithms, according to whether they guarantee finding the optimal solutions. Complete algorithms include distributed backtracking search \cite{modi2005adopt,chechetka2006no,yeoh2010bnb,netzer2012concurrent,litov2017forward,dengCCJL19} and inference \cite{petcuscalable,petcu2007mb,chen2020rmb}, which exhaust the solution space by either branch-and-bound or bucket elimination \cite{dechter98}. However, the coordination overheads of these algorithms scale exponentially w.r.t. the problem size, making them unsuitable for large-scale applications.

On the other hand, incomplete algorithms trade the optimality for practical computational overheads \cite{farinelli2008decentralised,cohen2020governing,chenDWH18,nguyen2019distributed,ottens2017duct}. Among them, local search \cite{Zhang2005Distributed,maheswaran2004distributed,Pearce2007Quality,zivan2014explorative,hoang2018large} is an important class of incomplete algorithms, which iteratively refines the solution via local moves. However, these algorithms often prematurely converge to poor local optima due to their greedy nature. As an instantiation of Guided Local Search (GLS) \cite{voudouris2010guided}, Generalized Distributed Breakout Algorithm (GDBA) \cite{okamoto2016distributed} was introduced to provide a comprehensive rule set for breaking out of the local optima by adapting the notion of constraint violation and  cost increase in DBA \cite{hirayama2005distributed}. However, the empirical benefits of GDBA often appear to be marginal compared to well-established baselines like DSA \cite{Zhang2005Distributed} on general-valued problems.

To understand why GDBA underperforms on general-valued problems, we first conduct a pilot study (cf.~Figure~\ref{fig:main}) analyzing the penalty dynamics of GDBA on both random DCOPs and structured problems. Our observations reveal that, on general-valued instances, GDBA uniformly accumulates penalties for nearly all constraints, making them receive heavy but similar attention, which offsets the benefit of breakout. Such ineffective penalization stems from over-aggressive constraint violation conditions (e.g., $NM$) that classify most constraints as violated, monotonic penalty increases that lead to unbounded penalty accumulation, and uncoordinated penalty updates that cause agents to optimize misaligned objective functions.
In light of this, we present a novel Distributed Guided Local Search (DGLS) framework for DCOPs, which incorporates an adaptive constraint violation condition based on the costs of each constraint, an evaporation mechanism to avoid unbounded penalty accumulation, and a synchronization scheme to enforce coordinated penalty updates. Specifically, our contributions are:
\begin{itemize}
    \item We systematically examine the key design choices of GDBA. We find that the over-aggressive constraint violation conditions, monotonic penalty increment and uncoordinated penalty update would lead to inferior performance on general-valued DCOPs.
    \item We present a novel DGLS framework that enables efficient GLS for DCOPs. Our DGLS incorporates an adaptive constraint violation condition, an evaporation mechanism, and a penalty synchronization scheme to fully unleash the performance of GLS in solving DCOPs. We also theoretically show that the penalty values are bounded, and agents play a potential game in our DGLS. 
    \item We compare our DGLS with state-of-the-art DCOP algorithms on various standard benchmarks. Our extensive empirical results show the great potential of DGLS: on general-valued problems and scale-free network, DGLS is able to match or perform slightly better than Damped Max-sum (DMS) \cite{cohen2020governing} with a high damping factor, while DGLS outperforms DMS by significant margins on structured problems including 2D lattices, meeting scheduling and weighted graph coloring in terms of anytime performance \cite{zilberstein1996using,zivan2014explorative}.
\end{itemize}
\section{Preliminaries}\label{sec:pre}
In this section, we review necessary preliminaries including DCOPs, GLS and GDBA.
\subsection{Distributed Constraint Optimization Problems}
A DCOP \cite{modi2005adopt} can be formalized as a tuple $\langle I,X,D,F\rangle$ where $I=\{1,\dots,|I|\}$ is the set of agents, $X=\{x_1,\dots,x_{|X|}\}$ is the set of variables, $D=\{D_1,\dots,D_{|X|}\}$ is the set of discrete domains, and $F=\{f_1,\dots,f_{|F|}\}$ is the set of constraint functions. Each variable $x_i$ takes a value from its domain $D_i$, and each constraint function $f_i:D_{i_1}\times\cdots \times D_{i_k}\rightarrow \mathbb{R}_{\ge 0}$ defines a cost for each possible combination of variables $scp(f_i)=(x_{i_1},\dots,x_{i_k})$. The objective is to find a solution $\tau^*=\left(d_1^*,\dots,d_{|X|}^*\right)$ such that the total cost is minimized:
\begin{equation}
    \tau^*=\mathop{\arg\min}_{\tau\in\prod_i D_i}\sum_{f_j\in F}f_j\left(\tau|_{scp(f_j)}\right),\label{eq:dcop_obj}
\end{equation}
where $\tau|_{scp(f_j)}$ is the projection of $\tau$ onto $scp(f_j)\subseteq X$. For the sake of simplicity, we follow the common assumptions that each agent controls only one variable (i.e., $|I|=|X|$) and all constraints are binary (i.e., $f_{ij}: D_i\times D_j\rightarrow\mathbb{R}_{\ge 0}, \forall f_{ij}\in F$). Therefore, the terms ``agent'' and ``variable'' can be used interchangeably.
\subsection{Guided Local Search}
GLS \cite{voudouris2010guided} is a metaheuristic for helping local search algorithms to escape local optima using a penalty system. Specifically, GLS considers the following augmented objective function:
\begin{equation}
    h(\tau)=f(\tau)+\lambda\sum_i p_i\cdot\mathbb{I}[\text{feature }i\text{ presents in } \tau], \label{eq:gls_obj}
\end{equation}
where $f(\cdot)$ is the original objective function, $p_i$ is the penalty associated with feature $i$, $\mathbb{I}$ is the indicator function, and $\lambda>0$ is the weight to balance the original objective and penalty. Here, the features are problem-specific. For example, a feature could be an edge from city A to city B in a Traveling Salesman Problem (TSP), or whether a hard clause is satisfied in Boolean Satisfiability (SAT) \cite{cai2020old}.
When local search converges to a local optimum, GLS will select a subset of features presented in the incumbent solution to increase the associated penalty value, so as to force local search to explore novel solutions. In other words, GLS breaks out of the local optimum by modifying the problem's objective landscape.
\subsection{Generalized Distributed Breakout Algorithm}
GDBA \cite{okamoto2016distributed} provides a comprehensive set of rules for breaking out of local optima in DCOPs, which adapts the concepts of constraint violation and cost increase in DBA \cite{hirayama2005distributed} for solving Distributed Constraint Satisfaction Problems (DisCSPs) \cite{yokoo1998distributed}. Essentially, GDBA can be viewed as an instantiation of GLS where the features can be a full constraint function, a specific row or column within it, or a single cost cell, depending on the scope of changes to the penalty values during breakouts. Specifically, when a Quasi Local Minimum (QLM) \cite{hirayama2005distributed} is detected, each agent in QLM first identifies a set of violated constraints based on the cost values under the incumbent local solution, then \emph{independently} increases the penalty associated with the features in these violated constraints by 1, which corresponds to selective penalization in GLS. Besides additive penalty like Eq.~(\ref{eq:gls_obj}), GDBA's variants also consider the multiplicative penalty:
\begin{equation}
    h(\tau)=\sum_{f_{ij}\in F}f_{ij}(d_i,d_j)\cdot\left[1+M_{ij}(d_i,d_j)\right],\label{eq:mul-obj}
\end{equation}
where $\tau=(d_1,\dots,d_{|X|})$ is a solution, $M_{ij}$ is the cost modifier (i.e., a matrix of penalty values) associated with constraint function $f_{ij}$ with the initial value of 0.
\section{Distributed Guided Local Search} \label{sec:main}
In this section, we present the Distributed Guided Local Search (DGLS) framework. We first motivate our research by analyzing the key design choices of GDBA. Then we detail DGLS and theoretically analyze its properties.
\subsection{Motivation}\label{sec:moti}
To analyze the dynamics of penalty values of GDBA, we consider its widely used variant $\langle M,NM,T\rangle$ on both sparse random DCOPs (RND) and meeting scheduling (MS) \cite{zivan2014explorative,maheswaran2004taking} where GDBA demonstrates notably inferior and superior performance, respectively (cf.~Section~\ref{sec:exp}). Figure~\ref{fig:sub1} plots the mean penalty values in all cost modifiers against iterations.
\begin{figure}[t]
    \centering
    \begin{subfigure}{0.495\linewidth}
        \centering
        \includegraphics[width=\linewidth]{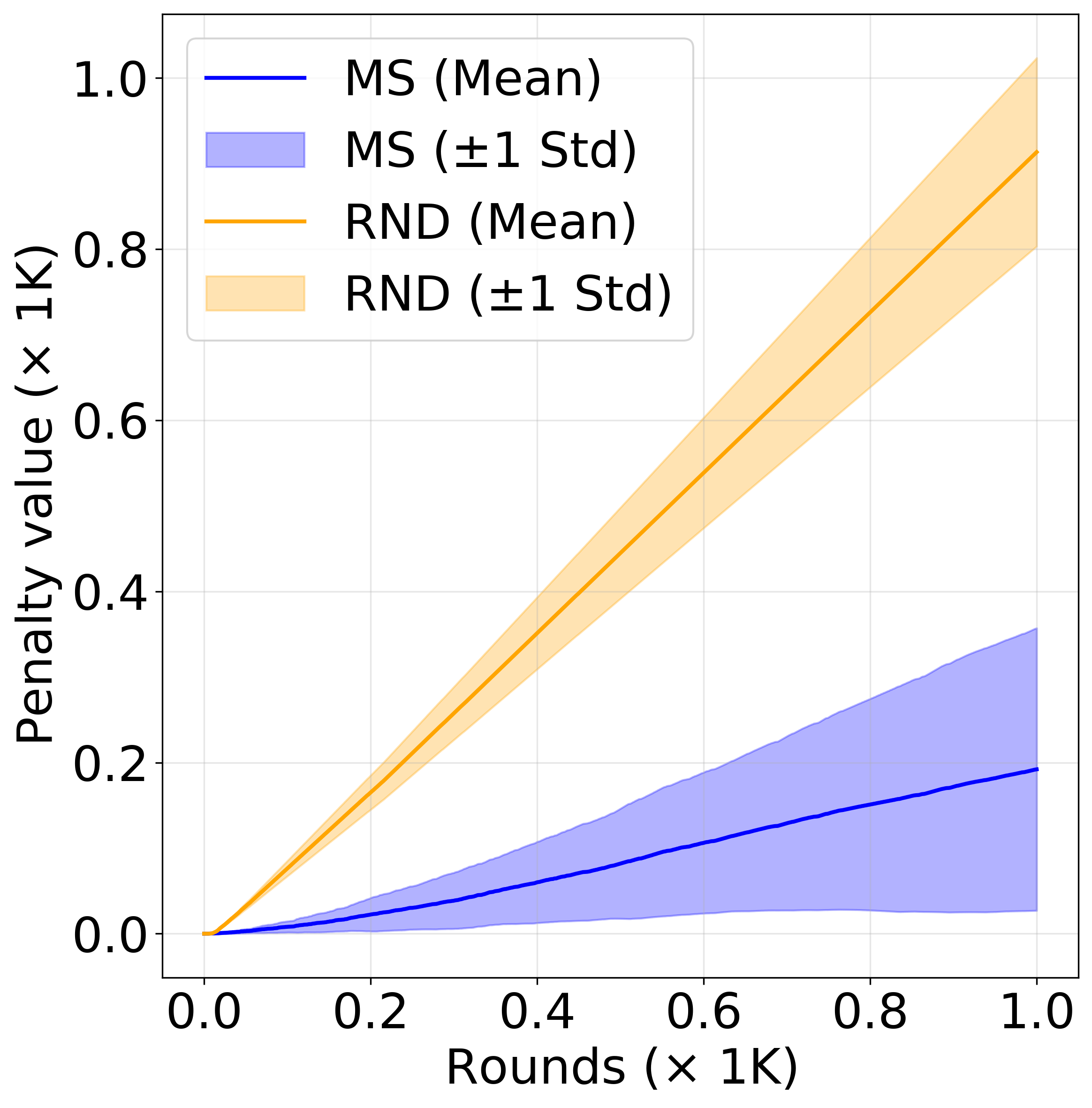}
        \caption{Mean of penalty values}
        \label{fig:sub1}
    \end{subfigure}
    \hfill 
    \begin{subfigure}{0.495\linewidth}
        \centering
        \includegraphics[width=\linewidth]{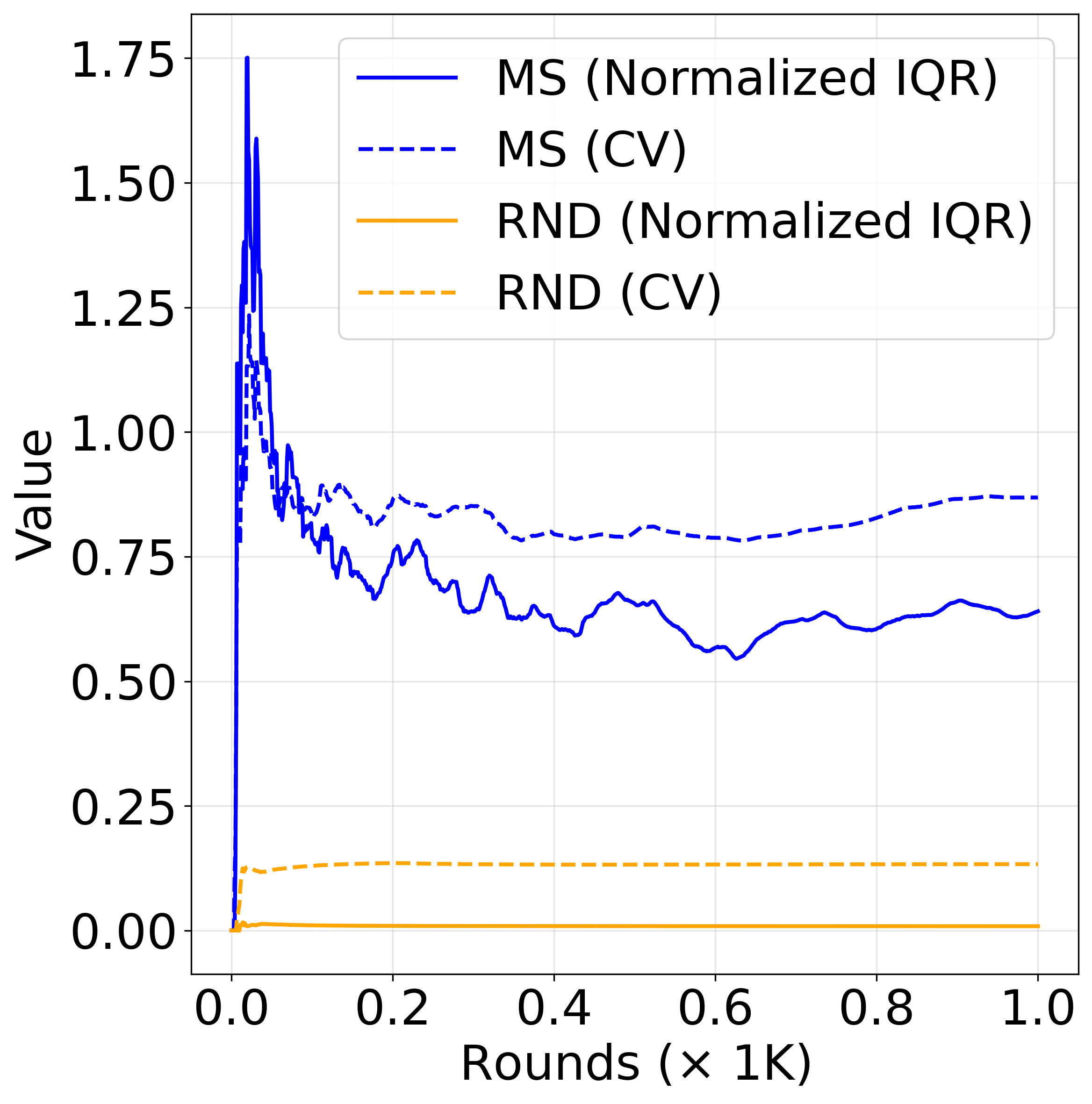}
        \caption{Variability of penalty values}
        \label{fig:sub2}
    \end{subfigure}
    \caption{Analysis of penalty dynamics of GDBA}
    \label{fig:main}
\end{figure}

It can be seen that the average penalty grows linearly in random DCOPs, meaning that almost every constraint is penalized in most rounds. That is due to the fact that the non-minimum ($NM$) violation condition is over-aggressive for general-valued constraints where the minimum cost entries are sparse. As a result, most constraints are considered as violated and the corresponding cost modifiers are increased once a QLM is detected. In contrast, when solving meeting scheduling problem where multiple minimum costs present in each constraint, penalty growth is much more moderate.

To look deeper into this issue, we also display the normalized IQR (i.e., interquartile range divided by mean) and Coefficient of Variation (CV) of penalty values in Figure~\ref{fig:sub2}. It can be observed that, compared to solving meeting scheduling problems, penalty values in GDBA have a much lower variability when solving random DCOPs. Such indiscriminate penalization cause each constraint to receive heavy but similar attention, which could offset the benefit of breakout.

The unbounded penalty accumulation also contributes to the ineffective penalization. Specifically, GDBA monotonically increases cost modifiers without any decay mechanism, even when the constraints are almost perfectly satisfied, e.g., constraint costs that are only slightly above the minimum costs, which in turn exacerbates the pathologies of over-penalization and indiscriminate penalization, as evidenced by Figure~\ref{fig:sub1} and \ref{fig:sub2}, respectively. 

Finally, agents in QLM independently increase the cost modifiers, which may cause agents to optimize different objectives. Consider a constraint $f_{ij}$ being flagged as violated for the first time. If agent $i$ is in QLM while agent $j$ doesn't, then the cost modifier for $f_{ij}$ becomes 1 from agent $i$’s side but remains 0 from agent $j$’s side after breakout. Such mismatch introduces asymmetry in cost structures and potentially breaks the correspondence between pure strategy Nash Equilibrium and local optimum \cite{chapman2011unifying,grinshpoun2013asymmetric}.
\subsection{DGLS Framework} \label{sec:dgls}
In light of this, we present a novel Distributed Guided Local Search (DGLS) framework for DCOPs. To overcome the over-penalization and indiscriminate penalization, our DGLS incorporates an adaptive rule for selectively penalizing the constraints with high costs. Besides, we introduce an evaporation mechanism to avoid unbounded penalty accumulation. Finally, we propose a synchronization scheme to enforce coordinated penalty update. Algorithm \ref{alg:dgls} presents the sketch of DGLS.
\begin{algorithm}[t]
\caption{Distributed Guided Local Search for agent $i$}
\label{alg:dgls}
\begin{algorithmic}[1]
\small
\State Initialize cost modifiers to 0
\State Choose a random assignment $d_i\in D_i$
\State Send $x_i=d_i$ to all neighbors $\mathcal{N}_i$
\While{termination condition is not met}
\State $\bar{P}_i\gets \emptyset$
\State Receive assignment $x_j=d_j$ from all neighbors $\mathcal{N}_i$
\State $d_i^*\gets\arg\min_{d\in D_i}\sum_j\textsc{EffCost}(d,j,d_j)$
\State $\Delta_i\gets\sum_j\textsc{EffCost}(d_i,j,d_j)-\textsc{EffCost}(d_i^*,j,d_j)$
\State Send gain $\Delta_i$ to all neighbors $\mathcal{N}_i$
\State Receive gain $\Delta_j$ from all neighbors $\mathcal{N}_i$
\If{$\Delta_i>0$}
\If{$\Delta_i$ is the best improvement}
\State $d_i\gets d_i^*$
\EndIf
\ElsIf{no neighbor can improve}
\For{$j\in \mathcal{N}_i$}
\If{\textsc{IsViolated}$(d_i,j,d_j)$}
\State $\bar{P}_i\gets \bar{P}_i\cup\{j\}$
\State Send a (SYNC, $i$) message to agent $j$
\EndIf
\EndFor
\EndIf 
\State $\tilde{P}_i\gets$ receive (SYNC, $j$) from neighbors $\mathcal{N}_i$
\For{$j\in\mathcal{N}_i$}
\State\textsc{Evaporate}$(j)$
\State\textsc{IncreaseMod}$(d_i,j,d_j,\bar{P}_i,\tilde{P}_i)$
\EndFor
\State Send $x_i=d_i$ to all neighbors $\mathcal{N}_i$
\EndWhile 
\end{algorithmic}
\end{algorithm}

Like GDBA, each agent $i$ in our DGLS begins with initializing the cost modifiers and broadcasting a randomly selected assignment $d_i$ to its neighbors. After that, it looks for the best assignment $d_i^*$ and calculates the corresponding gain $\Delta_i$ based on the current local view and cost modifiers via \textsc{EffCost} (cf.~Algorithm~\ref{alg:effcost}), where the base cost is modified in either additive ($A$) or multiplicative ($M$) way. After that, it broadcasts the gain to the neighbors to determine whether it should make a local move. Particularly, if $\Delta_i=0$ and no neighbor can improve, then a QLM is detected and agent $i$ starts to select constraints to penalize. After that, all cost modifiers are evaporated by a rate $\gamma$, and finally agent $i$ performs coordinated penalty update given a scope (i.e., cell, table, row or column). Therefore, a DGLS is instantiated by specifying a tuple $(A/M,\gamma,cel/tab/row/col)$.
\begin{algorithm}[t]
\caption{Effective Cost Computation}
\label{alg:effcost}
\begin{algorithmic}[1]
\Function{EffCost}{$d_i,j,d_j$}
\If{$manner=additive$}
\State\textbf{return} $f_{ij}(d_i,d_j)+M_{ij}(d_i,d_j)$
\ElsIf{$manner=multiplicative$}
\State\textbf{return} $f_{ij}(d_i,d_j)\cdot\left[M_{ij}(d_i,d_j) + 1\right]$
\EndIf
  \EndFunction
\end{algorithmic}
\end{algorithm}
\subsubsection{Adaptive violation condition} Instead of using fixed rules in GDBA, our DGLS incorporates an adaptive rule for identifying violated constraints (cf.~Algorithm~\ref{alg:adaptive}). Intuitively, for each constraint $f_{ij}$ we first compute its normalized cost $\eta$, where $\check{f}_{ij}=\min_{d_i,d_j}f_{ij}(d_i,d_j)$ and $\hat{f}_{ij}=\max_{d_i,d_j}f_{ij}(d_i,d_j)$. Then we stochastically mark the constraint as violated with a probability of $\eta$. This also generalizes the utility-based penalization in classical GLS \cite{voudouris2010guided}, where the features with the highest score are deterministically penalized.

Our adaptive rule has several unique advantages. First, it eliminates the need to tune the constraint violation condition in GDBA (e.g., $NZ,\;NM$ and $MX$). Second, it perfectly aligns with objective (\ref{eq:dcop_obj}) by measuring the ``badness" of a constraint with the normalized cost. Particularly, if the constraint cost equals the minimum value, then the constraint cannot be flagged as violated; conversely, if the constraint cost attains the maximum value, then there must be a constraint violation. By selectively penalizing constraints, we avoid over-penalizing the constraints with the cost close to their minimum value, which directs local search to pay more attention to the constraints that incur high costs.
\begin{algorithm}[t]
\caption{Adaptive Violation Condition}
\label{alg:adaptive}
\begin{algorithmic}[1]
\Function{IsViolated}{$d_i,j,d_j$}
    \State $\eta\gets \frac{f_{ij}(d_i,d_j)-\check{f}_{ij}}{\hat{f}_{ij}-\check{f}_{ij}}$
    \If{\textsc{Random}(0, 1)$<\eta$}
    \State \textbf{return} True
    \EndIf
    \State \textbf{return} False
  \EndFunction
\end{algorithmic}
\end{algorithm}
\subsubsection{Evaporation mechanism} We periodically decay the cost modifiers through an evaporation mechanism (cf.~\textsc{Evaporate}). For each constraint $f_{ij}$, the associated cost modifier $M_{ij}$ is geometrically decayed by $0<\gamma< 1$. Formally,
\begin{equation}
    M_{ij}(d_i,d_j)\gets \gamma M_{ij}(d_i,d_j),\quad\forall d_i\in D_i,d_j\in D_j.
\end{equation}
The evaporation mechanism addresses the issue of unbounded penalty accumulation. Combining with our adaptive violation condition, it helps local search to effectively forget the penalty of the well-satisfied constraints (e.g., the constraints whose current cost is close to the minimum value), and thus realizes selective penalization.

\subsubsection{Coordinated penalty update} We coordinate the penalty update between two agents by explicitly communicating the index of constraints to be penalized. Specifically, for each round, agent $i$ maintains a set of self-penalized constraints $\bar{P}_i$ and a set of constraints penalized by neighbors $\tilde{P}_i$. When it decides to penalize a constraint $f_{ij}$ (i.e., \textsc{IsViolated}$(d_i,j,d_j)$=True), agent $i$ first records the index to $\bar{P}_i$ and notify neighbor $j$ through a SYNC message. After that, it collects the index associated with all SYNC messages from neighbors as $\tilde{P}_i$. Finally, it performs coordinated penalty update according to Algorithm~\ref{alg:coordinated}.

Specifically, if the update scope is either cell ($cel$) or table ($tab$), then the corresponding entries of the cost modifier will be increased by 1 if the index presents in either $\bar{P}_i$ or $\tilde{P}_i$. On the other hand, if the update scope is row, agent $i$ will increase the entries in the $d_i$-th row regardless of $x_j$'s assignment if the penalization of $f_{ij}$ is initiated by agent $i$. If $j\in \tilde{P}_i$, agent $i$ mirrors the operation of agent $j$ by penalizing the $d_j$-th column of the cost modifier. Finally, if both agent $i$ and $j$ penalize constraint $f_{ij}$, then we minus 1 from $M_{ij}(d_i,d_j)$ to avoid double-counting. The scope of column ($col$) follows a similar pattern by exchanging row and column operations.
\begin{algorithm}[t]
\caption{Coordinated Penalty Update}
\label{alg:coordinated}
\begin{algorithmic}[1]
\Function{IncreaseMod}{$d_i,j,d_j,\bar{P}_i,\tilde{P}_i$}
\If{$scope=cell$}
\If{$j\in \bar{P}_i\lor j\in \tilde{P}_i$}
\State $M_{ij}(d_i,d_j)\gets M_{ij}(d_i,d_j)+1$
\EndIf
\ElsIf{$scope=table$}
\If{$j\in \bar{P}_i\lor j\in \tilde{P}_i$}
\State $M_{ij}(d_i^\prime,d_j^\prime)\gets M_{ij}(d_i^\prime,d_j^\prime)+1,\;\forall d_i^\prime,d_j^\prime$
\EndIf
\ElsIf{$scope=row$}
\If{$j\in \bar{P}_i$}
\State $M_{ij}(d_i,d_j^\prime)\gets M_{ij}(d_i,d_j^\prime)+1,\;\forall d_j^\prime$
\EndIf
\If{$j\in \tilde{P}_i$}
\State $M_{ij}(d_i^\prime,d_j)\gets M_{ij}(d_i^\prime,d_j)+1,\;\forall d_i^\prime$
\EndIf
\If{$j\in \bar{P}_i\land j\in \tilde{P}_i$}
\State $M_{ij}(d_i,d_j)\gets M_{ij}(d_i,d_j)-1$
\EndIf
\ElsIf{$scope=column$}
\If{$j\in \bar{P}_i$}
\State $M_{ij}(d_i^\prime,d_j)\gets M_{ij}(d_i^\prime,d_j)+1,\;\forall d_i^\prime$
\EndIf
\If{$j\in \tilde{P}_i$}
\State $M_{ij}(d_i,d_j^\prime)\gets M_{ij}(d_i,d_j^\prime)+1,\;\forall d_j^\prime$
\EndIf
\If{$j\in \bar{P}_i\land j\in \tilde{P}_i$}
\State $M_{ij}(d_i,d_j)\gets M_{ij}(d_i,d_j)-1$
\EndIf
\EndIf
\EndFunction
\end{algorithmic}
\end{algorithm}
\subsection{Theoretical Results}\label{sec:theoretical}
\newtheorem{lemma}{Lemma}
\newtheorem{theorem}{Theorem}
\newtheorem{corollary}{Corollary}
In this subsection, we theoretically analyze the properties of DGLS. We first show the upper bound of values in cost modifiers in our DGLS, which avoids the uncontrollably growing penalty values in GDBA.
\begin{theorem}
    With evaporation rate $0<\gamma<1$, the penalty values in any cost modifier are bounded by $1/(1-\gamma)$.
\end{theorem}
\begin{proof}
    Consider the worst-case scenario where a specific entry $(d_i,d_j)$ is incremented in every round. Let $M^{(k)}_{ij}$ be the penalty value in round $k$. We have
    \[
    \begin{aligned}
        M^{(1)}_{ij}(d_i,d_j)&=0\cdot\gamma +1=1\\
        M^{(2)}_{ij}(d_i,d_j)&=1\cdot\gamma +1=1+\gamma\\
        \cdots&\cdots\\
        M^{(k)}_{ij}(d_i,d_j)&=1+\gamma+\gamma^2+\cdots+\gamma^{k-1}.
    \end{aligned}
    \]
    As $k\to\infty$, this geometric series converges to $1/(1-\gamma)$, which concludes the theorem.
\end{proof}
\begin{corollary}
    The effective cost is bounded for both additive and multiplicative cases:
    \[
    \begin{aligned}
        \textnormal{\textsc{EffCost}}_A(d_i,j,d_j)&\leq\hat{f}_{ij}+1/(1-\gamma)\\
        \textnormal{\textsc{EffCost}}_M(d_i,j,d_j)&\leq\hat{f}_{ij}\cdot\left[1+1/(1-\gamma)\right].
    \end{aligned}
    \]
\end{corollary}
This corollary indicates that the contribution of any constraint to the total effective cost is bounded by a constant, which prevents any constraint from dominating the augmented objective as the number of rounds grows.

We then show that our coordinated penalty update guarantees the consistency of cost modifiers from both sides of each constraint. This consistency, in turn, enables a potential game structure in DGLS.
\begin{lemma}
    At the beginning of each round, the cost modifier of the constraint between agent $i$ and agent $j$ from $i$'s side is the same as the counterpart from $j$'s side.\footnote{Proof is provided in the extended version.}
\end{lemma}

\begin{theorem}
For each round, agents in DGLS play a potential game where the potential function is the total effective cost given the current cost modifiers.
\end{theorem}
\begin{proof}
    Define the potential function as
    \begin{equation}
        \Phi(\tau)=\frac{1}{2}\sum\nolimits_{i\in I}\sum\nolimits_{j\in\mathcal{N}_i}\textsc{EffCost}(d_i,j,d_j),\label{eq:potential_function}
    \end{equation}
    where $\tau=(d_1,\dots,d_{|X|})$ is a solution. Consider an agent $i$ unilaterally changing its assignment from $d_i$ to $d_i^\prime$, while all other agents maintain their current assignments. The change in agent $i$'s local cost is:
    \[
        \Delta_i=\sum\nolimits_{j\in\mathcal{N}_i}\textsc{EffCost}(d_i,j,d_j)-\textsc{EffCost}(d_i^\prime,j,d_j).
    \]
    The change in the potential function is:
    \[
        \begin{aligned}
        \Delta\Phi=&\frac{1}{2}\sum_{j\in\mathcal{N}_i}\textsc{EffCost}(d_i,j,d_j)-\textsc{EffCost}(d_i^\prime,j,d_j)\\
        &+\frac{1}{2}\sum_{j\in\mathcal{N}_i}\textsc{EffCost}(d_j,i,d_i)-\textsc{EffCost}(d_j,i,d_i^\prime).
        \end{aligned}
    \]
    By Lemma 1, the cost modifiers are the same from both agent $i$'s and $j$'s side, which implies
    \[
        \textsc{EffCost}(d_i,j,d_j)=\textsc{EffCost}(d_j,i,d_i),
    \]
    and 
    \[
        \textsc{EffCost}(d_i^\prime,j,d_j)=\textsc{EffCost}(d_j,i,d_i^\prime).
    \]
    Therefore, $\Delta_i=\Delta\Phi$, which concludes the theorem.
\end{proof}
Theorem 2 indicates that unlike in GDBA, agents in our DGLS optimize a consistent and well-defined global objective, i.e., the total effective cost (cf.~Eq.~(\ref{eq:gls_obj}) and Eq.~(\ref{eq:mul-obj})). Furthermore, any local improvement of an agent exactly corresponds to a reduction in the global effective cost.

We now establish some equivalences of DGLS variants.
\begin{theorem}
    If $f_{ij}:D_i\times D_j\to\{0,1\},\;\forall f_{ij}\in F$, then $(A,\gamma,cel)$ and $(M,\gamma,cel)$ are equivalent for any $\gamma$.
\end{theorem}
\begin{proof}
    We prove the theorem by showing that for each round $k$ and assignments $d_i\in D_i,d_j\in D_j$,
    \begin{equation}
        \textsc{EffCost}_A^{(k)}(d_i,j,d_j)=\textsc{EffCost}_M^{(k)}(d_i,j,d_j),\label{eq:eff_eq}
    \end{equation}
    where $\textsc{EffCost}_A^{(k)}$ and $\textsc{EffCost}_M^{(k)}$ are the effective cost for round $k$ under additive and multiplicative manner, respectively. If $f_{ij}(d_i,d_j)=0$, then the constraint $f_{ij}$ cannot be flagged as violated since $\eta=0$ (cf.~Algorithm~\ref{alg:adaptive}). Therefore the corresponding entry $(d_i,d_j)$ of the cost modifier $M_{ij}$ in both variants remains 0 since no increment happens. In this case, Eq.~(\ref{eq:eff_eq}) holds because
    \[
        \begin{aligned}
        &\textsc{EffCost}_A^{(k)}(d_i,j,d_j)=0+M_{ij}(d_i,d_j)=0+0=0\\
        &\textsc{EffCost}_M^{(k)}(d_i,j,d_j)=0\cdot\left[1+M_{ij}(d_i,d_j)\right]=0\cdot 1=0.
        \end{aligned}
    \]
    
    On the other hand, if $f_{ij}(d_i,d_j)=1$, then the two variants maintain the same penalty for entry $(d_i,d_j)$ given the same decay factor $\gamma$. We show it by induction. The fact holds for the first round where the cost modifiers are 0. Assume that it holds for round $k$. If there is a QLM and $d_i,d_j$ are the incumbent assignments, then $f_{ij}$ must be violated and the penalty of entry $(d_i,d_j)$ will be incremented by 1 in both variants since $\eta=1$. Otherwise, there is no increment for both variants. Given the same evaporation rate $\gamma$, the fact holds for round $k+1$. Therefore, Eq.~(\ref{eq:eff_eq}) holds for this case because
    \[
        \resizebox{.99\linewidth}{!}{$
        \begin{aligned}
        &\textsc{EffCost}_A^{(k)}(d_i,j,d_j)=1+M_{ij}(d_i,d_j)\\
        &\textsc{EffCost}_M^{(k)}(d_i,j,d_j)=1\cdot\left[1+M_{ij}(d_i,d_j)\right]=1+M_{ij}(d_i,d_j).
        \end{aligned}
        $}
    \]
    
\end{proof}
\begin{theorem}
$(A,\gamma,tab)$ is equivalent to Maximum Gain Message (MGM) for any $\gamma$.    
\end{theorem}
\begin{proof}
    Let's consider the decision process of agent $i$ under the additive manner:
    \[
    \begin{aligned}
        d_i^*&=\mathop{\arg\min}\nolimits_{d_i^\prime\in D_i}\sum\nolimits_{j\in\mathcal{N}_i}\textnormal{\textsc{EffCost}}(d_i^\prime,j,d_j)\\
        &=\mathop{\arg\min}\nolimits_{d_i^\prime\in D_i}\sum\nolimits_{j\in\mathcal{N}_i}f_{ij}(d_i^\prime,d_j)+M_{ij}(d_i^\prime,d_j).
    \end{aligned}
    \]
    In table scope, the cost modifiers are updated table-wise (cf.~line 5-7 of Algorithm~4). Therefore $M_{ij}(d_i^\prime,d_j)=c_{ij},\;\forall d_i^\prime \in D_i$, which gives us
    \[
    \begin{aligned}
        d_i^*
        &=\mathop{\arg\min}\nolimits_{d_i^\prime\in D_i}\sum\nolimits_{j\in\mathcal{N}_i}f_{ij}(d_i^\prime,d_j)+c_{ij}\\
        &=\mathop{\arg\min}\nolimits_{d_i^\prime\in D_i}\left(c+\sum\nolimits_{j\in\mathcal{N}_i}f_{ij}(d_i^\prime,d_j)\right)\\
        &=\mathop{\arg\min}\nolimits_{d_i^\prime\in D_i}\sum\nolimits_{j\in\mathcal{N}_i}f_{ij}(d_i^\prime,d_j).
    \end{aligned}
    \]
    This exactly matches the decision rule for agents in MGM \cite{maheswaran2004distributed}. Besides,
    \[
    \resizebox{.99\linewidth}{!}{$
    \begin{aligned}
    \Delta_i&=\sum_{j\in\mathcal{N}_i}\textnormal{\textsc{EffCost}}(d_i,j,d_j)-\textnormal{\textsc{EffCost}}(d_i^*,j,d_j)\\
    &=\sum_{j\in\mathcal{N}_i}f_{ij}(d_i,d_j)+M_{ij}(d_i,d_j)-f_{ij}(d_i^*,d_j)-M_{ij}(d_i^*,d_j)\\
    &=\sum_{j\in\mathcal{N}_i}f_{ij}(d_i,d_j)-f_{ij}(d_i^*,d_j)+c_{ij}-c_{ij}\\
    &=\sum_{j\in\mathcal{N}_i}f_{ij}(d_i,d_j)-f_{ij}(d_i^*,d_j),
    \end{aligned}
    $}
    \]
    which also produces the same gain as in MGM. Therefore, the theorem is concluded.
\end{proof}

Finally, we show the complexity of DGLS as follows.
\begin{theorem}
    In each round of DGLS, agent $i$ communicates $O(|\mathcal{N}_i|)$ messages and performs $O\left(|\mathcal{N}_i|*|D^i_{\max}|*|D_i|\right)$ operations, where $D^i_{\max}=\arg\max_{j\in\mathcal{N}_i}|D_j|$.
\end{theorem}
\begin{proof}
    Each round agent $i$ communicates $|\mathcal{N}_i|$ assignment messages, $|\mathcal{N}_i|$ gain messages, and $q$ SYNC messages, where $q\leq |\mathcal{N}_i|$ since $i$ notifies a neighbor $j$ only when $f_{ij}$ is flagged as violated. Therefore, the total number of messages communicated by agent $i$ is in $O(|\mathcal{N}_i|)$.

    Besides, agent $i$ finds $d_i^*$ in $O(|\mathcal{N}_i|*|D_i|)$ operations, evaporates all cost modifiers in $O\left(|\mathcal{N}_i|*|D^i_{\max}|*|D_i|\right)$ operations and increments the cost modifiers in $O\left(|\mathcal{N}_i|*|D^i_{\max}|*|D_i|\right)$ operations in the worst case (e.g., $tab$ scope is used and all involved constraints are penalized). Therefore, the total number of operations is in $O\left(|\mathcal{N}_i|*|D^i_{\max}|*|D_i|\right)$.
\end{proof}
\begin{figure}[t]
    \centering
    \begin{subfigure}{0.495\linewidth}
        \centering
        \includegraphics[width=\linewidth]{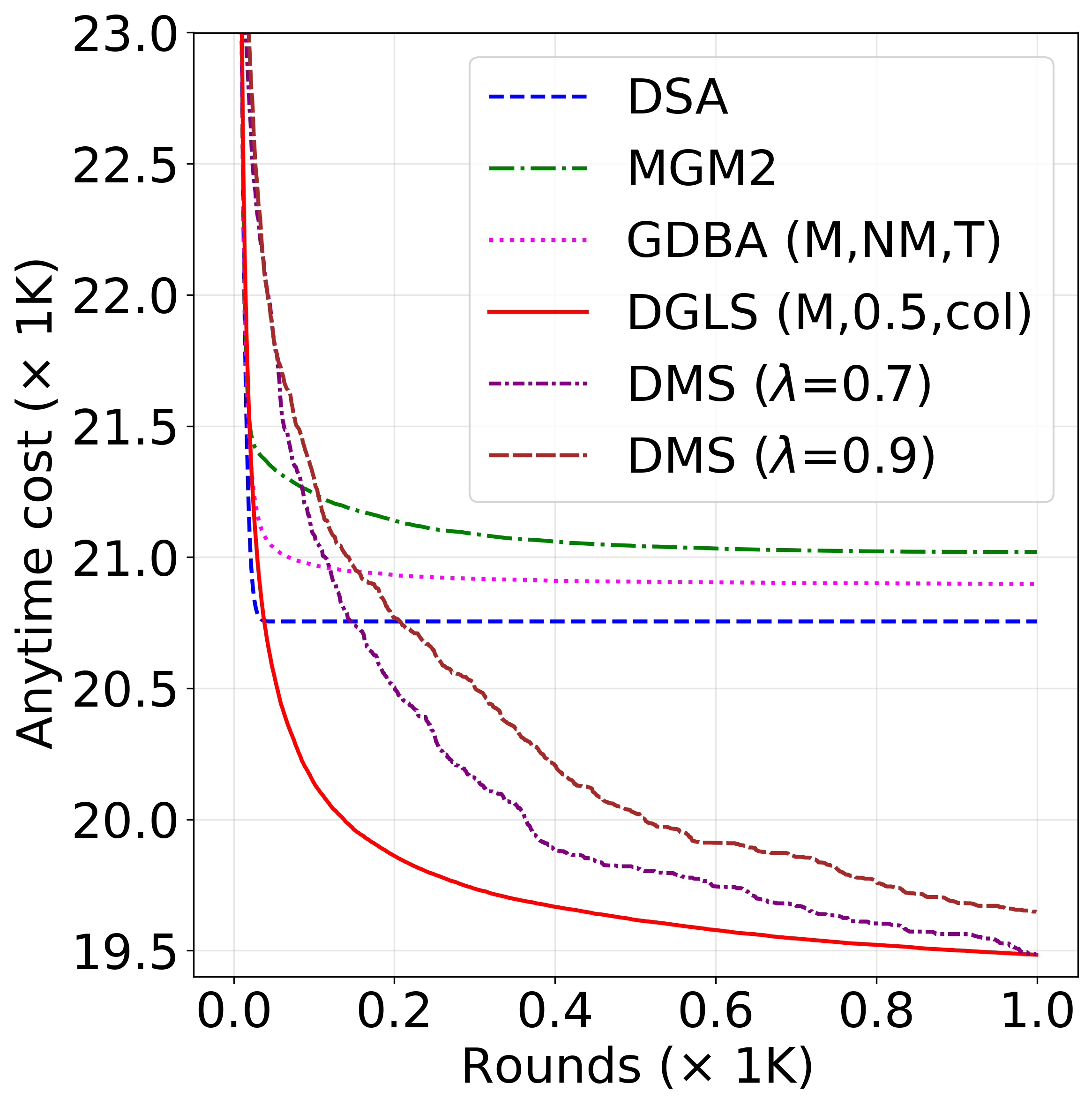}
        \caption{Sparse problems}
        \label{fig:sparse_dcop}
    \end{subfigure}
    \hfill 
    \begin{subfigure}{0.495\linewidth}
        \centering
        \includegraphics[width=\linewidth]{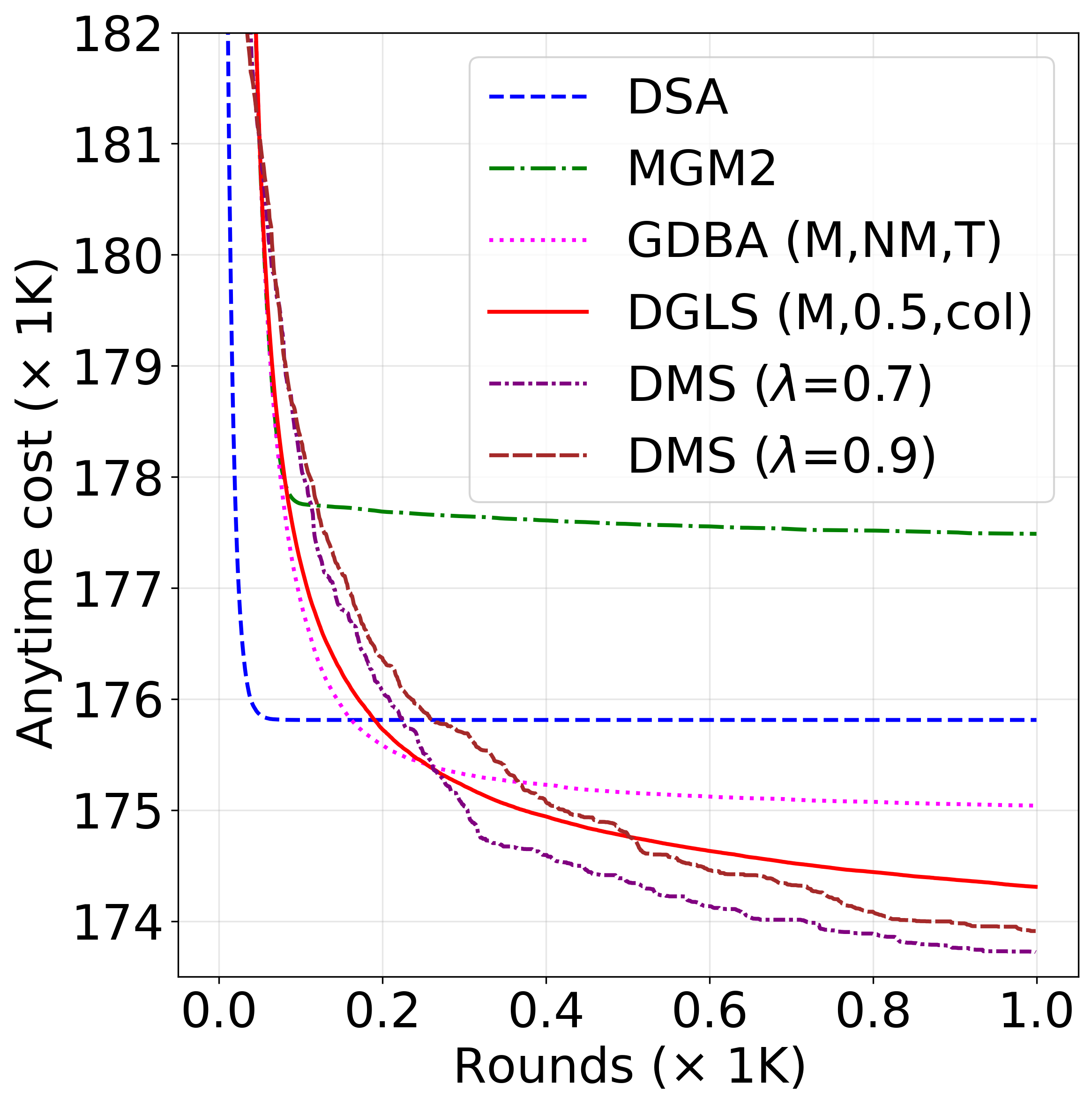}
        \caption{Dense problems}
        \label{fig:dense_dcop}
    \end{subfigure}
    \caption{Performance on random DCOPs}
    \label{fig:random_dcop}
\end{figure}
\section{Empirical Evaluations}\label{sec:exp}
\subsubsection{Benchmarks and baselines} We evaluate our algorithm on various standard DCOP benchmarks, including: (1) random DCOPs with 120 agents, a graph density of 0.1 (sparse) or 0.6 (dense); (2) scale-free networks \cite{barabasi1999emergence} with 120 agents and $m_0=m_1=3$; (3) 2D lattices with grid size of $10\times 10$; (4) meeting scheduling problems using Events-as-Variable (EAV) formulation \cite{zivan2014explorative,maheswaran2004taking} with 20 available time slots, 20 meetings and 90 persons, where each person randomly selects two meetings to participate, the travel time between any pair of meetings is uniformly drawn from $[6,10]$, and a cost equal to the number of overbooked persons is incurred if the difference between the time slots of two meetings are less than the travel time; (5) weighted graph coloring problems with 120 agents, 3 available colors, and graph density of 0.05, where a conflict cost is uniformly selected from $[1,100]$ if two adjacent agents assign the same color. For benchmarks (1--3), we consider the problems with a domain size of 10 and uniformly select constraint costs from $[0,100]$. For each benchmark, we generate 100 random instances and each instance is solved 20 times with the maximum round of 1000. Finally, we average the anytime cost \cite{zilberstein1996using,zivan2014explorative} of each round of 2000 runs as the final result. All experiments are conducted on a Linux workstation with Intel Xeon W-2133 CPU and 32GB memory.

For competitors, we consider DSA \cite{Zhang2005Distributed} with $p=0.8$ and GDBA \cite{okamoto2016distributed} as representative local search algorithms; MGM2 \cite{maheswaran2004distributed} as a representative $K$-OPT algorithm, and DMS \cite{cohen2020governing} with damping factors $\lambda=0.7$ and 0.9 as the state-of-the-art baselines. For GDBA, we consider its $(M,NM,T)$ variant that exhibits the strongest performance according to \cite{okamoto2016distributed}.
\begin{figure}[t]
    \centering
    \begin{subfigure}{0.495\linewidth}
        \centering
        \includegraphics[width=\linewidth]{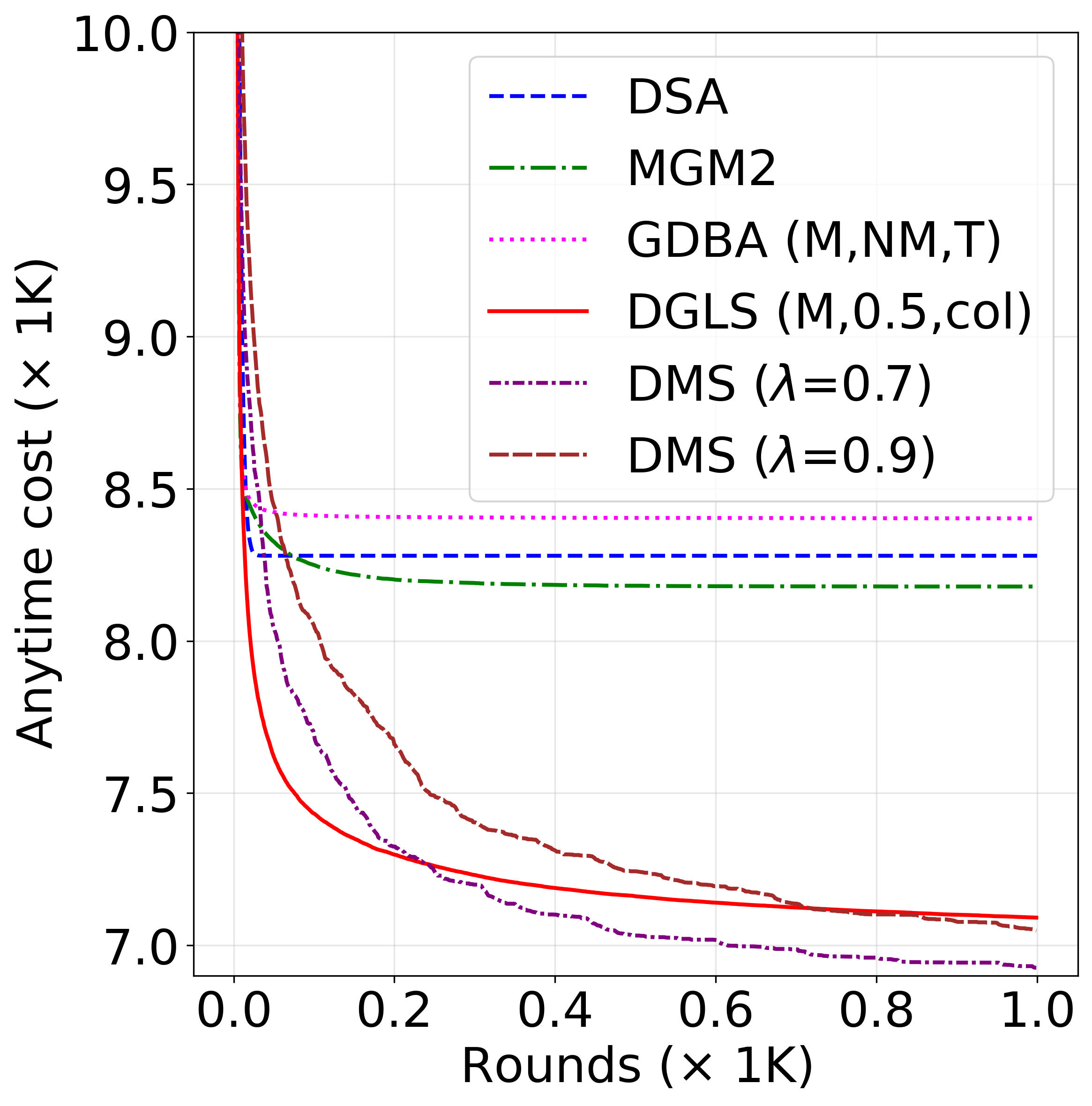}
        \caption{Scale-free networks}
        \label{fig:sf_net}
    \end{subfigure}
    \hfill 
    \begin{subfigure}{0.495\linewidth}
        \centering
        \includegraphics[width=\linewidth]{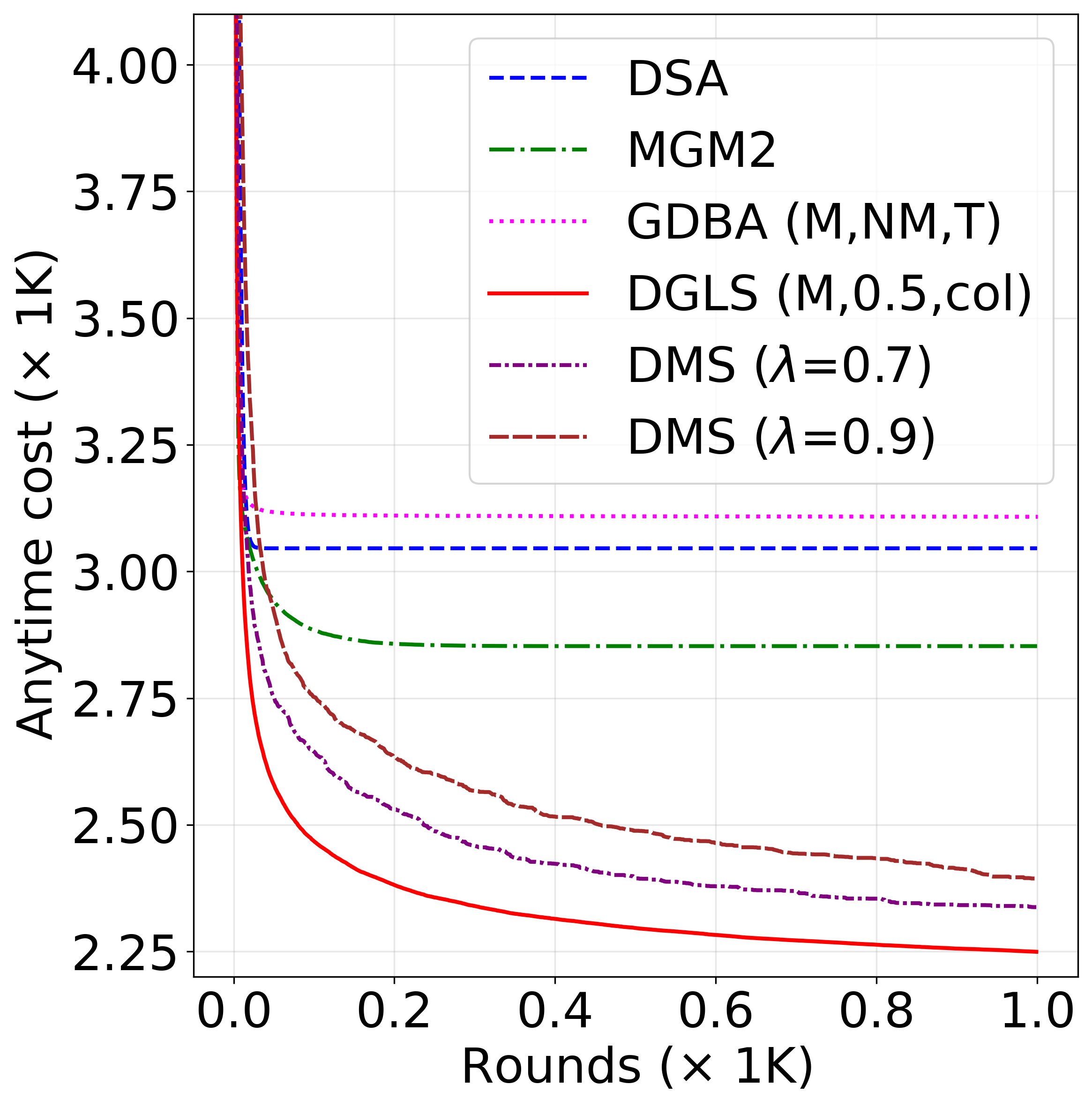}
        \caption{2D lattices}
        \label{fig:2d}
    \end{subfigure}
    \caption{Performance on topology-structured problems}
    \label{fig:topology}
\end{figure}
\subsubsection{Performance comparison} Empirically, we find DGLS variants $(M,0.5,col)$ and $(M,0.9,col)$ perform best on general-valued problems and cost-structured problems, respectively. Therefore, we consider these two variants for performance comparison. Figure~\ref{fig:random_dcop} presents the anytime results on both sparse and dense random DCOPs. It can be observed that our DGLS exhibits a fast convergence speed on sparse problems, quickly surpasses all baselines after about 50 rounds. GDBA, on the other hand, fails to effectively break out of local optima and is dominated by DSA in the sparse case. Besides, it is interesting to find that the gap between GDBA and our DGLS narrows on dense problems. This is because each constraint has a higher cost than the sparse setting. Therefore, agents in our DGLS trigger penalization more frequently (cf.~Algorithm~\ref{alg:adaptive}), which exhibits similar behavior to the GDBA under the $NM$ violation condition. Still, DGLS outperforms GDBA by a significant margin in the dense setting, thanks to the evaporation mechanism and synchronization scheme for coordinated penalty updates.

Figure~\ref{fig:topology} compares the anytime performance on topology-structured problems. It can be seen that GDBA performs poorly and is strictly dominated by all other competitors on both scale-free networks and 2D lattices, which highlights the inefficiency of GDBA in dealing with general-valued problems. In contrast, our DGLS effectively changes the problems' landscape through adaptive violation condition, evaporation and coordinated penalty update whenever local search gets trapped in local optima. These mechanisms enable selective penalization that directs local search to focus more on the constraints with high costs, and therefore result in a much steadier improvement over time. In fact, DGLS not only significantly outperforms all local search heuristics, but also improves DMS by 3.77\%--6.03\% on 2D lattices with $p$-value$<10^{-5}$.

Figure~\ref{fig:cost} presents the results on weighted graph-coloring problems and meeting scheduling problems. GDBA demonstrates great advantages over all competitors except DGLS on these cost-structured problems. This can be attributed to the multiple cost minimums present in each constraint (e.g., the cost entries for conflict-free colors in weighted graph coloring problems or time slots in meeting scheduling problems).  As a consequence, the cost modifiers in GDBA grow moderately compared to general-valued problems (cf.~Figure~\ref{fig:main}), thus the over-penalization and indiscriminate penalization are alleviated as a side effect. Nevertheless, GDBA is still strictly dominated by DGLS which explicitly implements selective penalization through the adaptive violation condition and evaporation mechanism. Notably, our DGLS surpasses all baselines within the first 50 rounds and continuously improves given more rounds, significantly outperforming DMS by 61.24\%--66.30\% and 5.47\%--9.45\% on weighted graph-coloring problems and meeting scheduling problems with $p$-value$<10^{-5}$, respectively.
\begin{figure}[t]
    \centering
    \begin{subfigure}{0.495\linewidth}
        \centering
        \includegraphics[width=\linewidth]{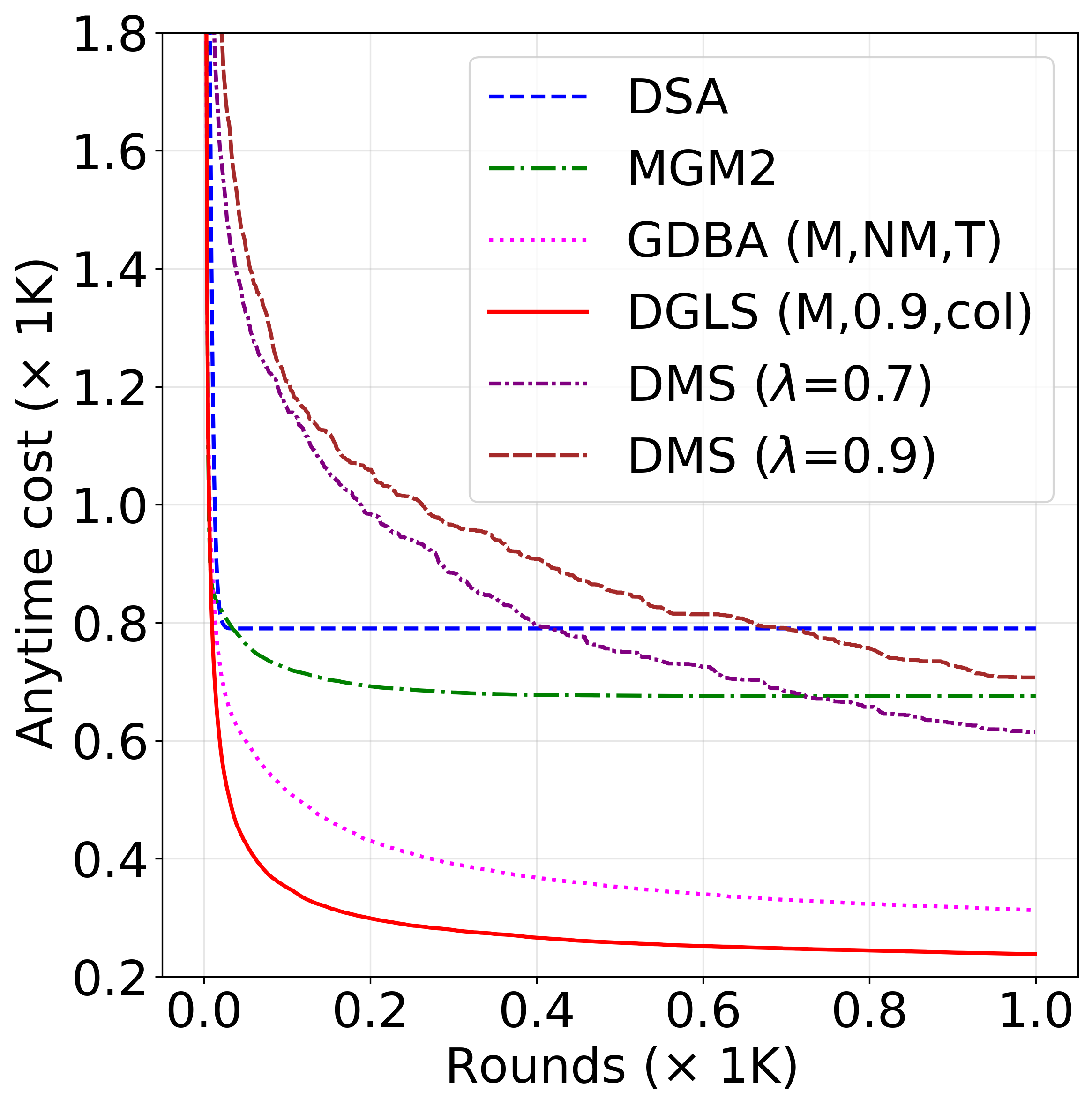}
        \caption{Weighted graph-coloring}
        \label{fig:wgc}
    \end{subfigure}
    \hfill 
    \begin{subfigure}{0.495\linewidth}
        \centering
        \includegraphics[width=\linewidth]{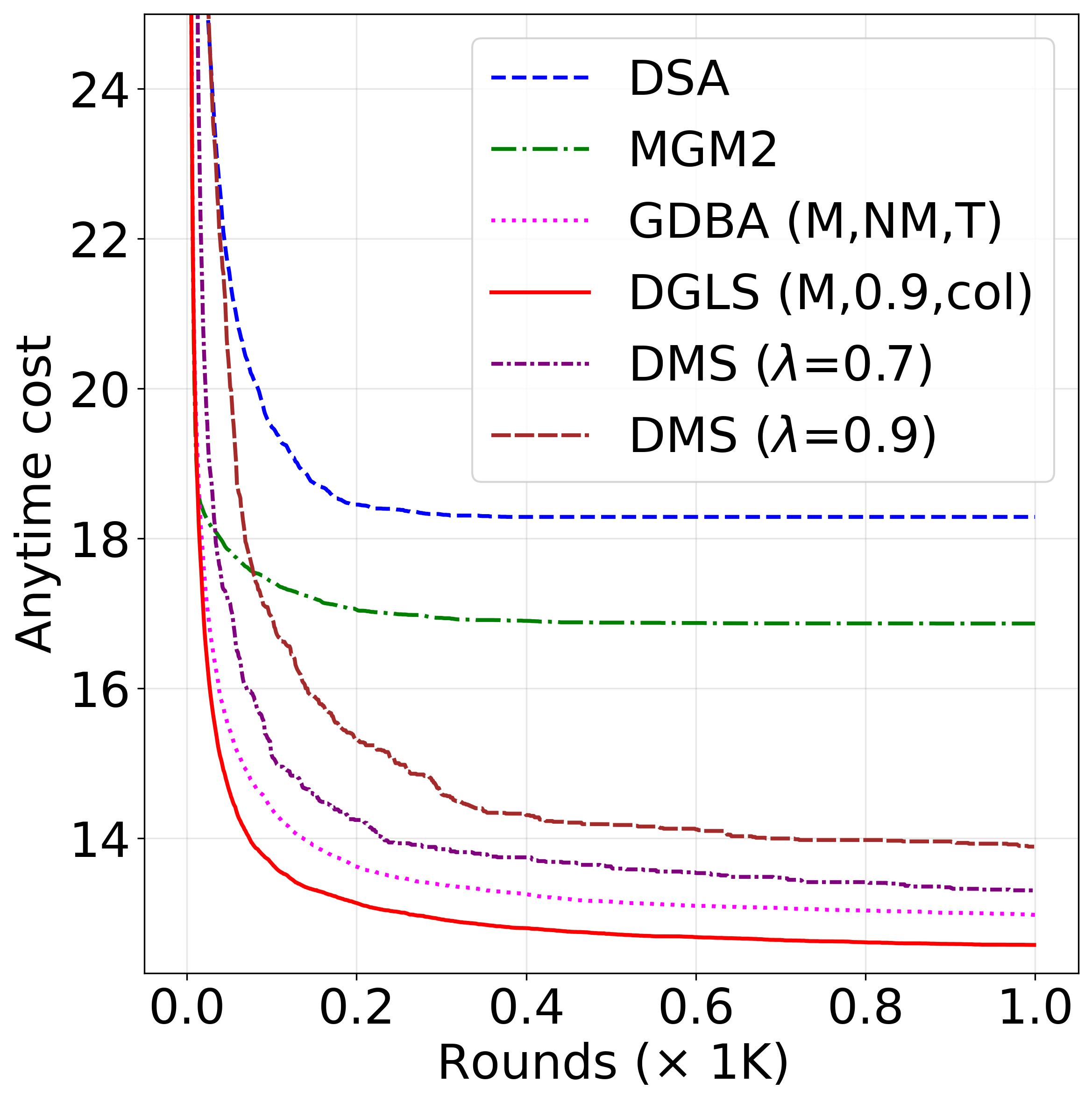}
        \caption{Meeting scheduling}
        \label{fig:ms}
    \end{subfigure}
    \caption{Performance on cost-structured problems}
    \label{fig:cost}
\end{figure}
\subsubsection{Ablation study} To understand how each component contributes to the success of our DGLS, we perform an ablation study on sparse random DCOPs and present results in Figure~\ref{fig:ablation}. Here, we consider the DGLS variant of $(M,0.5,tab)$ for ablation since it is algorithmically comparable to GDBA $(M,NM,T)$ due to the same manner and scope parameters. In fact, GDBA $(M,NM,T)$ can be recovered from DGLS $(M,0.5,tab)$ if adaptive violation condition (AVC), evaporation, and coordinated penalty update (CPU) are disabled. 

AVC tends to contribute most significantly to the anytime performance. Without AVC, DGLS indiscriminately penalizes each constraint that incurs a cost higher than the minimum cost, which leads to ineffective penalization and performs only slightly better than DSA. This also aligns with our observation that the over-aggressive constraint violation conditions in GDBA could severely limit its performance on general-valued DCOPs.
Evaporation, on the other hand, also plays an important role in controlling the magnitude of the penalty values. Without it, DGLS suffers from unbounded penalty growth and substantially slower convergence compared to the full DGLS. Interestingly, compared to DGLS w/o AVC, DGLS w/o evaporation exhibits better performance, which also highlights the advantage of our adaptive violation over the fixed violation rules in GDBA. Nonetheless, the combination of AVC and evaporation works synergistically to enable effective selective penalization, as demonstrated by the strong performance of DGLS w/o CPU and full DGLS. Finally, CPU ensures that agents optimize coherently given the changing cost modifiers, leading to moderate but consistent improvements.
\begin{figure}
    \centering
    \includegraphics[width=0.8\linewidth]{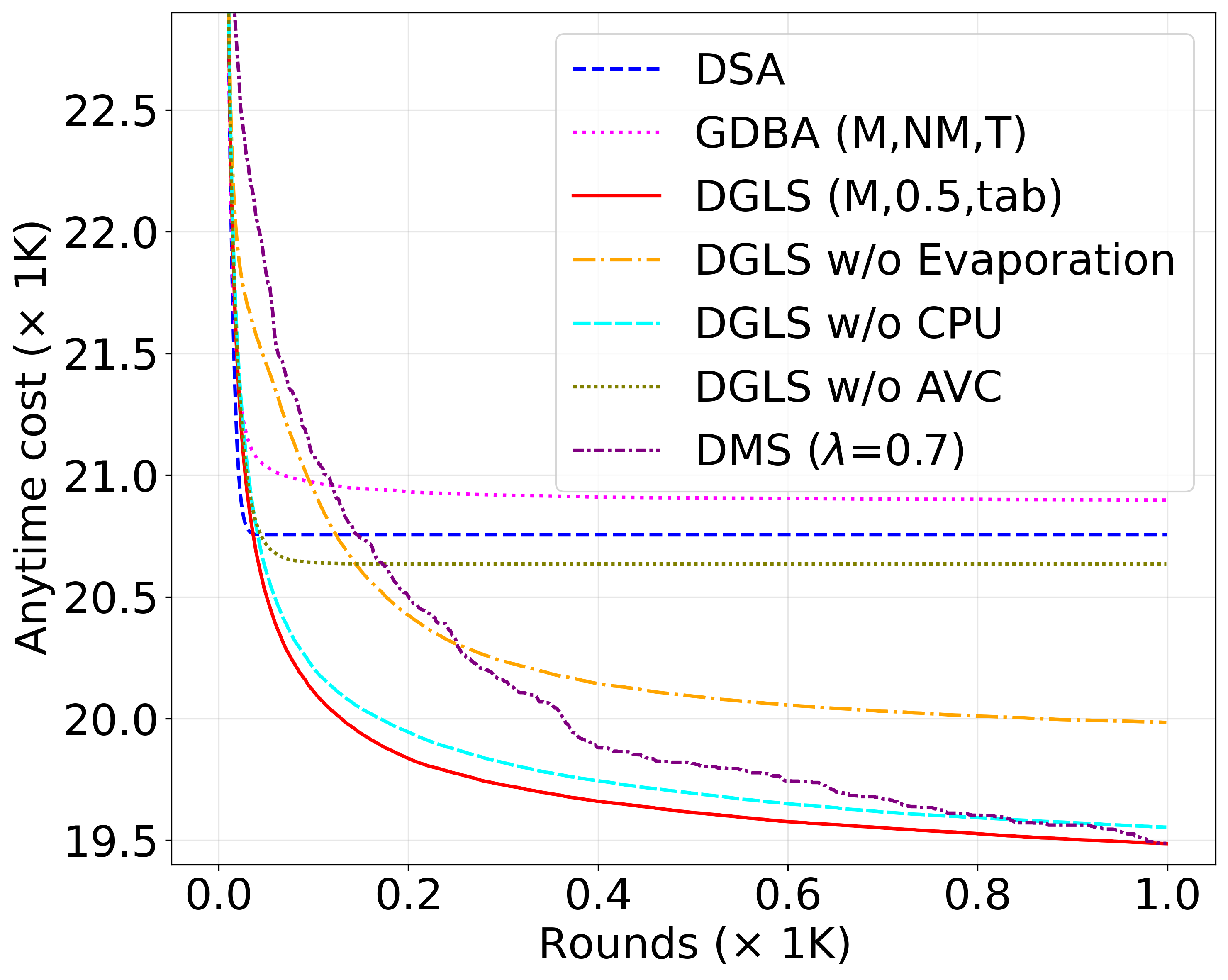}
    \caption{Ablation study on sparse random DCOPs}
    \label{fig:ablation}
\end{figure}
\section{Conclusion}
As an instantiation of GLS on DCOPs, GDBA aims to help local search break out of local optima by increasing the penalty values of violated constraints. However, its empirical benefits remain marginal on general-valued problems. This suboptimal performance stems from three key issues: over-aggressive constraint violation conditions, unbounded penalty accumulation, and uncoordinated penalty updates. Such pathologies lead to ubiquitous over-penalization and indiscriminate penalization, ultimately undermining the intended benefits of breakout.

We therefore present DGLS, a novel GLS framework for DCOPs that effectively addresses these issues by incorporating an adaptive violation condition to selectively penalize constraints with high cost, a penalty evaporation mechanism to control the magnitude of penalization, and a synchronization scheme for coordinated penalty updates. Theoretically, we show that the penalty values of our DGLS are bounded, and agents play a potential game where the potential function is the total augmented cost given the current cost modifiers. Our extensive empirical evaluations on various standard benchmarks confirm the superiority of DGLS over the existing local search heuristics, as well as the state-of-the-art Damped Max-sum on both general-valued and cost-structured problems.
\section*{Acknowledgments}
This research is supported by the Ministry of Education, Singapore, under its MOE AcRF Tier 2 Award MOE-T2EP20223-0003. Xinrun Wang is supported by Singapore Ministry of Education (MOE) Academic Research
Fund (AcRF) Tier 1 grant (No. MSS24C005). Any opinions, findings and conclusions or recommendations expressed in this material are those of the author(s) and do not reflect the views of the Ministry of Education, Singapore.

\bibliography{aaai2026}

\end{document}